\documentclass[sn-apa,pdflatex]{sn-jnl}

\RequireXeTeX
\bibpunct{(}{)}{;}{a}{,}{,}
\makeatletter
\renewcommand{\@evenhead}{}
\renewcommand{\@oddhead}{}
\renewcommand{\ps@headings}{}
\renewcommand{\ps@titlepage}{}
\renewcommand{\cite}{\citep}
\makeatother

\jyear{2026}%
\usepackage{pdflscape}
\usepackage{xcolor}
\usepackage{amsmath,amsfonts}
\usepackage{array}
\usepackage[caption=false,font=normalsize,labelfont=sf,textfont=sf]{subfig}
\usepackage{textcomp}
\usepackage{stfloats}
\usepackage{url}
\usepackage{verbatim}
\usepackage{graphicx}

\usepackage[australian]{babel}

\usepackage[utf8]{inputenc}
\usepackage{afterpage}
\usepackage{blkarray}
\usepackage{bigstrut}
\usepackage{booktabs}
\usepackage{amsfonts}
\usepackage{amsmath}
\usepackage{amssymb}
\usepackage{amsthm}
\usepackage{array}
\usepackage{enumitem}
\usepackage{etoolbox}
\usepackage{hhline}
\usepackage{mathtools}
\usepackage{placeins}
\usepackage{tabularx}
\usepackage{verbatim}
\usepackage{multirow}

\newcolumntype{R}[1]{>{\raggedleft\let\newline\\\arraybackslash\hspace{0pt}}m{#1}}

\theoremstyle{definition}
\newtheorem{theorem}{Theorem}

\newtheorem{proposition}[theorem]{Proposition}

\begin{document}

\title{\Large Lumbermark: Resistant~Clustering~by~Chopping~Up  Mutual~Reachability Minimum Spanning Trees}

\author*[1,2]{\fnm{Marek} \sur{Gagolewski}}\email{marek.gagolewski@pw.edu.pl}

\affil[1]{\orgname{Warsaw University of Technology}, \orgdiv{Faculty of Mathematics and Information Science},  \orgaddress{\street{ul.~Koszykowa 75}, \postcode{00-662} \city{Warsaw}, \country{Poland}}}

\affil[2]{\orgname{Systems Research Institute, Polish Academy of Sciences},
\orgaddress{\street{ul.~Newelska 6}, \postcode{01-447} \city{Warsaw}, \country{Poland}}}

\abstract{
We introduce Lumbermark, a robust divisive clustering algorithm capable of detecting clusters of varying sizes, densities, and shapes.  Lumbermark iteratively chops off large limbs connected by protruding segments of a dataset's mutual reachability minimum spanning tree.  The use of mutual reachability distances smoothens the data distribution and decreases the influence of low-density objects, such as noise points between clusters or outliers at their peripheries.  The algorithm can be viewed as a divisive version of Genie and an alternative to HDBSCAN that produces partitions with user-specified sizes.  A~fast, easy-to-use implementation of this new method is available in the open-source \textbf{lumbermark} package for Python and R.  We show that Lumbermark performs well on benchmark data and hope it will prove useful to data scientists and practitioners across different fields.
}


\maketitle
\clearpage

\section{Introduction}\label{sec:intro}

Clustering is an unsupervised learning task that aims to detect
high-density subregions within a multidimensional distribution
which a dataset is assumed to follow
\cite{WierzchonKlopotek:clusterspringer,BlumHopcroftKannan2020:foundds,JaegerBanks2023:clustanalrev}.
It allows data scientists, engineers, and decision makers to come up with valuable insights,
decrease the information overload, or build more tailored predictive models.
Successful applications of clustering are innumerable in
software testing \cite{AJORLOO2024111805}, %
manufacturing \cite{Dogan2021}, %
industry \cite{Jin2023}, %
or traffic management \cite{Wei2024}, %
to name a few.

We will focus on methods that detect crisp partitions, which work best when
true clusters are non-overlapping. However, it should be noted that also
soft procedures exist, e.g., fuzzy c-means \cite{Bezdek} or the EM algorithm for Gaussian mixtures that allows
separating multidimensional normal distributions.
Creators of clustering algorithms face many challenges for the problem at hand is, overall, underdetermined \cite{Hennig2015:trueclust,VanMechelenETAL2023:whitepaperclustbench}.
Each development can be viewed from multiple angles.
\begin{enumerate}[label=\alph*)]
\setlength\itemsep{0em}
\item Does the algorithm output a known-in-advance number of clusters $k$
for every $k$, or does it ``guess'' the number of clusters based on the
underlying parameters (like \textit{min\_cluster\_size} in DBSCAN) and the local
density of points; see~\cite{11471034}?  %
\item Is it hierarchical? In other words, does it yield partitions
that are properly nested (like in agglomerative or divisive approaches), or do its outputs drastically change when $k$ is varied (like in $k$-means)?
\item Does it scale well with the increased sample size if data are sampled from the same mixture distribution, or does it require the re-adjustment of its underlying parameters then?
\item Does it allow clusters of different shapes, especially if they are well-separated and sitting on different lower-dimensional submanifolds (compare, e.g., spectral clustering), or do they have to be centred around some pivot point (like in $k$-means or $k$-medoids)?
\item Does it allow clusters of different sizes? What is the highest and lowest level
    of cluster size inequality that it accepts? What is the smallest allowable cluster?
\item Does it allow clusters of different densities (which can be related to varying spread/variance or different subsample sizes)?
\item Is it robust in the presence of outliers and noise points? Does it tend to
    treat outliers as singleton clusters (like single linkage)?
\end{enumerate}

Many standalone clustering algorithms are proposed in the literature each year,
e.g., \cite{ldpmst,YANG2026114147,ZHAO2025113880,MA2021194,GagolewskiETAL2016:genie,itm,11340753},
amongst which HDBSCAN \cite{hdbscan} stands out as a particularly successful
recent development, enjoying sound theoretical foundations \cite{fasthdbscan}
and many variations, e.g., \cite{Jin2023,Wang2026,Luo2025},  %
as well as providing solid implementations for Python or R, which are widely
used by many practitioners.  HDBSCAN operates on a dataset's mutual reachability
minimum spanning tree (MST), which has the nice property that it pulls
peripheral points away from each other.  Unfortunately, even though it is based
on a robust version of single linkage, the algorithm in its original setting does not allow specifying
the output cluster count.

In \cite{GagolewskiETAL2025:cvimst}, a substantial number of $k$-clustering approaches
based on Euclidean minimum spanning trees has recently been reviewed.
In the current contribution, we extend the framework to the case
of mutual reachability MSTs, by first pointing out and then resolving their ambiguity (Sec.~\ref{sec:def}).
This allows us to introduce Lumbermark (Sec.~\ref{sec:lumbermark}),
a new hierarchical divisive algorithm that outputs crisp partitions
of any cardinality, is robust in presence of outliers, and detects clusters of varying shapes and sizes.
We will show (Sec.~\ref{sec:experiments}) that it outperforms the previous winner
of the MST-based clustering algorithms ranking, Genie \cite{GagolewskiETAL2016:genie}.

\section{Definitions and properties}
\label{sec:def}

Let $\mathbf{X}\in\mathbb{R}^{n\times d}$ define $n$ points described
by $d$ real-valued features,
$\mathbf{x}_1,\dots,\mathbf{x}_n\in \mathbb{R}^d$,
and let $\mathcal{X}=\{1,\dots,n\}$ be comprised of their corresponding indexes.
As usual in practice, we assume that data were subject to some appropriate feature engineering, e.g.,
scaling or selection of meaningful subsets of input variables.
By $d(i,j)=\| \mathbf{x}_i - \mathbf{x}_j \|$ we denote the Euclidean distance
between two given points. We assume that there are no two point pairs
enjoying the same Euclidean distance; otherwise, a tiny bit of noise could be added
to $\mathbf{X}$ to make the distances unambiguous.

\subsection{DBSCAN*}

\citet{hdbscan_lnai} and \citet{hdbscan}~introduced a simplified
version of the famous Density-Based Spatial Clustering
of Applications with Noise algorithm by \citet{dbscan}, which they named DBSCAN*.
Let us recall its workings.

For a fixed \emph{radius} $\varepsilon>0$ and a \emph{smoothing parameter} $M\ge 1$,
let $N_\varepsilon(i)=\{ j\neq i: d(i, j)\le\varepsilon \}$ be the
$\varepsilon$-neighbourhood of a point $i$. We call $i$ a \emph{core} point
if $\vert N_{\varepsilon}(i) \vert \ge M$, i.e., there are at least $M$ objects\footnote{%
    In their original paper, \citet{hdbscan_lnai} included
    the point $i$ itself in its own neighbourhood.
}
within radius $\varepsilon$ of $i$. Otherwise, that is if the $\varepsilon$-neighbourhood
of $i$ is not sufficiently ``dense'', it is labelled a \emph{noise} point%
\footnote{%
    The original DBSCAN algorithm \cite{dbscan} additionally distinguishes
    between noise and \emph{border points}, that is, non-core objects that are
    contained within $\varepsilon$-neighbourhoods of core points.
    Such a simplification actually makes DBSCAN* even more similar to its
    predecessor, the $(M, r)$-clustering method proposed by \citet{predbscan}.
}.

Each DBSCAN* cluster ${C}\subseteq\mathcal{X}$ is a maximal set of \emph{core} points such that
every pair $i,j\in{C}$ is \emph{density-connected},
i.e., there exists a path $i=p_1,p_2,\dots,p_{l-1},p_l=j\in C$
such that $p_{k+1}\in N_\varepsilon(p_{k})$ for all relevant $k$.

\subsection{Mutual reachability distance and the DBSCAN* algorithm}

Finding a meaningful radius $\varepsilon$ without analysing the
distribution of the pairwise distances may be difficult.
A very elegant reformulation of DBSCAN* was additionally introduced
by \citet{hdbscan_lnai,hdbscan} to facilitate the tuning of this parameter.

Let $d(i, (m))=\| \mathbf{x}_i-\mathbf{x}_{e(i, (m))} \|$ be the distance between $\mathbf{x}_i$ and its
$m$-th nearest neighbour w.r.t.~$d$, denoted $e(i, (m))$.
For a fixed $M\ge 1$, let the corresponding \emph{$M$-core distance}
of the $i$-th point be given by $c_M(i)=d(i, (M))$.
We define the \emph{$M$-mutual reachability distance} as a symmetric version
of a discrepancy measure from the paper \cite{optics}:
\begin{eqnarray}
d_M(i,j) &=& \max\{ d(i,j), c_M(i), c_M(j) \} \\
&=& \max\left\{ d(i,j), d(i, (M)), d(j, (M)) \right\}
\nonumber
\end{eqnarray}
for $i\neq j$ and $0$ otherwise. For any $M$, $d_M$ on $\mathcal{X}$ is a metric,
i.e., in particular, it fulfils the triangle inequality.
For $M = 1$, $d_M$ on $\mathcal{X}$ coincides with $d$.
Otherwise, $d_M$ pulls certain point pairs away from each other,
especially close neighbours in sparsely populated subregions
where the core distances $c_M$ tend to be high.

The following result connects the mutual reachability distances
with the aforementioned clustering algorithm.

\begin{proposition}
For a fixed $\varepsilon>0$ and a smoothing parameter $M\ge 1$,
Let the graph $G=(\mathcal{X}, E)$ consist of edges $\{i,j\}$ such that $d_M(i,j)\le \varepsilon$.
DBSCAN* finds the connected components of $G$, marks all singletons as noise points,
and treats all the remaining connected components as standalone clusters
comprised of core points.
\end{proposition}

\begin{proof}[Proof (sketch)]
If $c_M(i) > \varepsilon$,  %
then $\vert N_\varepsilon(i) \vert < M$ and hence $i$ is a noise point.
Also, it implies that $d_M(i,j)>\varepsilon$ for all
$j\neq i$ and thus it is an isolated vertex in $G$.
On the other hand, if $c_M(i) \le \varepsilon$, it is a core point.
By definition, $C$ is a connected component of $G$ if
it is the maximal subset of $\mathcal{X}$ such that
for all $i,j\in C$ there exists a path between them.
\end{proof}

\subsection{DBSCAN* from a computational perspective}

The DBSCAN* algorithm has nice statistical (finding level sets of an unknown
density function; compare \citealp{Hartigan1975}) as well as topological
(persistent homology) underpinnings as explained, e.g., by~\citet{fasthdbscan}.
In the current work, we are interested in a computational (algorithmic)
take on finding prominent clusters.  Notably, \citet{hdbscan_lnai,hdbscan} showed that DBSCAN* is equivalent to:
\begin{enumerate}[itemsep=0pt]
\item running the Single Linkage clustering algorithm on $\mathcal{X}$
with respect to the $M$-mutual reachability distance, $d_M$,

\item cutting the output dendrogram at a given level $\varepsilon>0$, and

\item treating the resulting nontrivial branches as standalone clusters
whilst marking the remaining singletons\footnote{%
The original HDBSCAN* paper distinguishes between the cases $M=0$ and $M=1$
such that the former does not mark singletons as noise points,
but treats them as degenerate clusters. This way, if we assume $e(i, (0))=i$
and hence $d(i, (0))=0$, we can treat it as a generalisation of the standard
single linkage scheme.
}
as noise points.
\end{enumerate}
The proof that for a given $\varepsilon>0$, we get exactly the same results
as with the original formulation of DBSCAN* is a by-product
of Theorem~\ref{thm:dbscanmst} that we present below.
Most importantly, the first step is now independent of $\varepsilon$,
which greatly simplifies and speeds up the whole algorithm.
Moreover, this formulation opens up space for many new offshoots thereof.
In particular, \citet{hdbscan} propose the HDBSCAN (hierarchical DBSCAN*)
algorithm that prevents the formation of too small point groups,
as controlled by the \textit{min\_cluster\_size} parameter,
For that, its popular implementations enclosed in the \textbf{scikit-learn}
\cite{sklearn}, \textbf{hdbscan} \cite{hdbscanpkg}, or \textbf{fast\_hdbscan}
Python packages use a robustified single linkage algorithm considered by
\citet{ChaudhuriDasgupta2010:robustsinglelinkage}.  They also implement
some heuristics to find the dendrogram cuts that lead to the ``most stable'' clusters.

\subsection{DBSCAN* and mutual reachability minimum spanning trees}

Nevertheless, even if all pairwise distances are unique, as assumed, it may happen that $d_M(i, (j))=d_M(i, (j'))$
for $j\neq j'$. Actually, this is frequently the case because of the way the
mutual reachability distances are defined: especially points with high $c_M$s
will have many of their near-neighbours w.r.t.~$d_M$ equidistant from them.
Because of this ambiguity, the above procedure makes us often unable
to obtain clusterings of all desired cardinalities:
by adjusting $\varepsilon$, the number of clusters does not necessarily increase/decrease by 1.
In particular, neither of the aforementioned Python implementations of HDBSCAN allow us to
request a specific number of clusters.

Furthermore, that to compute the single linkage clustering it is sufficient
to consider a minimum spanning tree of the pairwise distance graph and merge
clusters in increasing order of MST edge weights (or remove edges in decreasing
order and consider the resulting connected components) has already been pointed
out by \citet{GowerRoss1969:singlelinkagemst}. Nevertheless, we must take note
of the fact that in the case of mutual reachability
distances, due to ties, there can be many $M$-mutual reachability minimum
spanning trees, i.e., acyclic graphs on $\mathcal{X}$ with edges $\{i,j\}$
weighted by $d_M(i,j)$ minimising the total sum of distances
(an example is provided in the sequel).  Fortunately, neither the baseline
DBSCAN* algorithm nor HDBSCAN's robust single linkage is affected by such
an ambiguity, as established by the following result.

\begin{theorem}\label{thm:dbscanmst}
Let $G=(\mathcal{X}, E)$ be a complete  undirected graph
with edges $\{i,j\}$ weighted by $d_M(i,j)$, possibly tied.
For any $\varepsilon>0$, denote with $E\vert_\varepsilon=\{ \{i,j\}\in E: d_M(i,j)\le \varepsilon\}$.
Then the connected components of $G\vert_\varepsilon=(\mathcal{X},E\vert_\varepsilon)$
are the same as the connected components of
$T\vert_\varepsilon=(\mathcal{X},E'\vert_\varepsilon)$ for any minimum spanning tree
$T=(\mathcal{X},E')$ of $G$ and any $\varepsilon>0$.
\end{theorem}
\begin{proof}[Proof (sketch)] It suffices to study the workings of the
\citet{Kruskal1956:mst} minimum spanning tree algorithm up to some point.
We consider the connected components
that arise when we browse through the edges in nondecreasing order
of weights and stop merging the disjoint sets at the last edge $\{i,j\}$
such that $d_M(i,j)\le\varepsilon$. Additionally, we note that the connected
components (contrary to the tree structure per se) are not affected
by the order of the consumption of the edges with tied weights: if an edge
is skipped, it means that there already was a path between $i$ and $j$
in the current tree (and we want to avoid forming a cycle).
\end{proof}

\subsection{Addressing the ambiguity of mutual reachability minimum spanning trees}

Different $d_M$-minimum spanning trees may have very diverse
topologies; see Figure~\ref{fig:example-mutreach-mst-chains}.
Even though, by the result above, DBSCAN* and HDBSCAN are not affected by the
ambiguity,
if we wish to run other clustering algorithms that operate on MSTs alone,
this feature becomes very problematic.

In particular, \citet{GagolewskiETAL2025:cvimst} review many clustering
algorithms that yield a specific number of clusters $k$ which correspond to
the connected components of the input spanning tree whose $k-1$ chosen edges
were omitted. As the time to compute the MST is between $\Omega(n\log n)$
and $O(n^2)$ (depending on the dimensionality of the dataset), these algorithms can be particularly effective, especially when two natural greedy strategies are applied:
\begin{itemize}
\item divisive: remove $k-1$ edges from the spanning tree one by one based on some decision rule/criterion,
\item agglomerative: extend the graph with $n$ isolated vertices
by adding $n-k$ edges from the tree one by one.
\end{itemize}

\clearpage

\begin{figure*}[t!]
\centering
\includegraphics[width=1.0\linewidth]{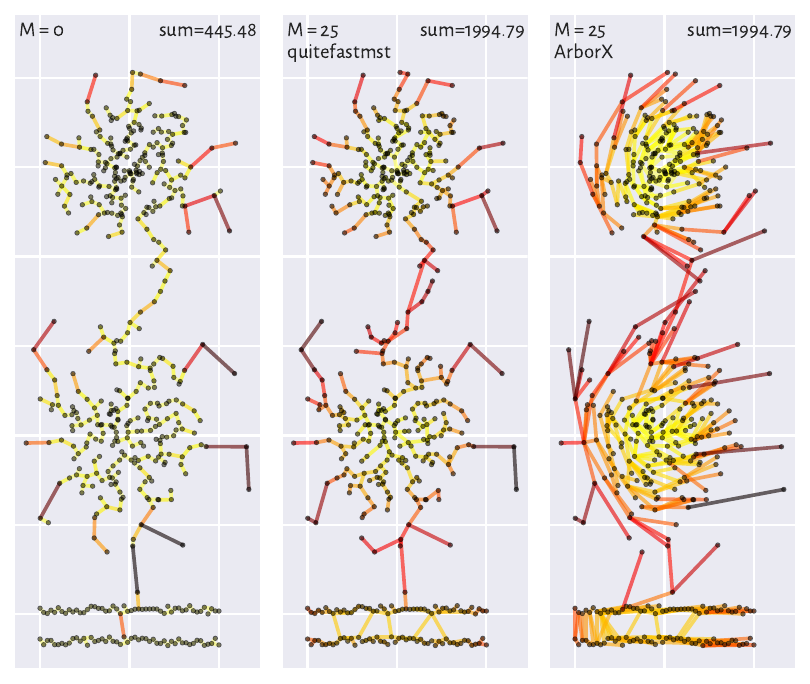}

\caption{\label{fig:example-mutreach-mst-chains}
Three copies of an example dataset with four clusters:
two horizontal segments and two Gaussian blobs connected by a point chain.
Additionally, its Euclidean minimum spanning tree (left) as well as
two 25-mutual reachability minimum spanning trees (centre, right) are depicted.
Edge colour intensities correspond to the underlying distances with the darkest
edge being the longest in each case.
By increasing the smoothing parameter $M$, we make the blobs better separable
but the close-lying segments become blended.
There are many more possible minimum spanning trees for $d_{25}$
and they can have very diverse topologies.
In the current paper, we adjust the mutual reachability distance
so that forming edges between point pairs with smaller original (Euclidean)
distances is preferred (Eq.~\eqref{eq:adjustedmutreach}; centre).
}
\end{figure*}

For example:
\begin{itemize}
\item Single linkage \cite{singlelinkage} adds $n-k$ shortest edges to the tree,
or, equivalently, removes $k-1$ longest ones \cite{GowerRoss1969:singlelinkagemst},
\item ITM \cite{itm} iteratively removes edges from the tree
that maximise an information-theoretic criterion based on entropy,
\item Genie \cite{GagolewskiETAL2016:genie} adds shortest edges one by one
provided that the Gini index of cluster sizes is below a given threshold;
otherwise, i.e., if the inequality of the cluster size distribution becomes too
high, the smallest cluster is merged with its nearest neighbour from the tree.
\end{itemize}

The mutual reachability distance is potentially attractive because
it smoothens out a noisy point distribution by making points in lower density
regions (on cluster boundaries, between clusters) more isolated,
because they tend to have larger $c_M$.
It is thus worthwhile to become capable of defining minimum spanning trees
with respect to $d_M$ unambiguously: it will enable us to make
the clustering algorithms normally applied on plain Euclidean MSTs potentially
more robust (a proposition which we will will test below).

\bigskip
Let us look at the building of the spanning tree from the perspective
of Borůvka's algorithm which, iteratively, for any given connected component $C$,
adds to the tree the shortest edge incident on a vertex in $C$ and on a vertex
not in $C$.
An ambiguity arises when we have multiple edges of the same weight, which
can happen if mutual reachability distances are clamped by some core distance.
For instance, let $i\in C$ and $j,j'\not\in C$ be such that
$d_M(i,j)=d_M(i,j')=c_M(i)$.
One possible approach would be to select $j$ if $c_M(j)<c_M(j')$
as $1/c_M(j)$ can be though of as a density estimate around $j$.
This way, points of lower core distances would receive more neighbours
in the tree, leading to the formation of bigger hubs.
Unfortunately, there can be many points of the same core distances,
so this method does not resolve the ties unambiguously.

Instead, we can prefer $j$ over $j'$ if the original distance
is $d(i,j)<d(i,j')$. This leads to trees that are more similar to
their Euclidean counterparts (compare the centre and the left plots
in Figure~\ref{fig:example-mutreach-mst-chains}).
Therefore, our new adjusted mutual reachability distance is
given for $i\neq j$ by:
\begin{equation}\label{eq:adjustedmutreach}
d_M'(i, j)=\max\{d(i, j), c_M(i), c_M(j)\}+\delta d(i, j),
\end{equation}
where $\delta>0$ is an infinitesimal constant.
It can easily be shown that if $\delta<\min_{i\neq j} d(i, j)$, then
$d_M'$ is a metric on $\mathcal{X}$.

Algorithmically, we can consider the original $d_M$ but with a binary less-than
comparer implemented in such a way that the closer pair sorts
before the farther one if ties are detected.
Our \textbf{quitefastmst}\footnote{See \url{https://quitefastmst.gagolewski.com/}.}  v0.9.2
package for Python and R implements this
approach exactly in the case of \textit{algorithm='brute'}, which relies on a parallelised Prim's
$O(n^2)$ method \cite{prim,Olson1995:parallelhierclust};
see centre plot in Figure~\ref{fig:example-mutreach-mst-chains}
(average path length$=55.25$).

Other implementations of mutual reachability spanning trees
yield similar trees, mainly because they reuse the $M$ nearest
neighbours needed to determine the core distances in the order
from the closest to the farthest. These are:
\textbf{hdbscan}\footnote{See \url{https://hdbscan.readthedocs.io/}.} v0.8.42 \cite{hdbscanpkg} (average path length$=55.75$),
\textbf{fast\_hdbscan}\footnote{See \url{https://github.com/TutteInstitute/fast_hdbscan/}.} v0.3.2 \cite{fasthdbscan} ($=55.20$),
\textbf{MemoGFK}\footnote{See \url{https://github.com/wangyiqiu/hdbscan}.} commit=e120dcf \cite{wangyiqiu_hdbscan}  ($=54.20$)
and \textbf{quitefastmst} v0.9.2 with K-d trees ($=55.42$).
On the other hand, \textbf{ArborX}\footnote{See \url{https://github.com/arborx/ArborX}.} commit=463fd9 \cite{arborx,arborx_emst}
outputs a tree of a very different topology; see the
right plot in Figure~\ref{fig:example-mutreach-mst-chains}
(average path length$=50.51$).

The topological properties of different mutual reachability
minimum spanning trees will be studied in a forthcoming paper
as this issue certainly deserves a more focused treatment.
Here we will rely on the natural adjusted distance given by Eq.~\eqref{eq:adjustedmutreach}.
In Section~\ref{sec:experiments}, we will be interested
in studying the effect of choosing the number of neighbours $M$
on the performance of two algorithms: Genie and Lumbermark,
the latter of which we introduce below.

\section{Lumbermark}\label{sec:lumbermark}

We propose the following divisive algorithm that fits the MST-based clustering framework discussed by \citet{GagolewskiETAL2025:cvimst}.
It operates under the key assumption that edges of a spanning tree that connect
points in the same cluster are of small weights whilst those connecting different
clusters are relatively longer.
However, outliers can also be present; they are represented by small
branches at the tree peripheries and should not be considered standalone clusters.

\subsection{The algorithm}

\begin{algorithm}[h!]
\caption{\textbf{Lumbermark}$(T, k; f)$:}
\label{algo:lumbermark}

\noindent
Input: An $M$-mutual reachability minimum spanning tree $T$ of $\mathcal{X}$,
the requested number of clusters (\textit{n\_clusters}) $k$.

\noindent
Tunable parameter:
(\textit{min\_cluster\_factor}) $f\in[0,1]$.

\noindent
Output: a $k$-clustering of $\mathcal{X}$.

\begin{enumerate}
\item[1.] Let $T'$ be a version of $T$ with all the leaves removed.

\item[2.] Let (\textit{cut\_edges}) $C=\emptyset$.

\item[3.] While $\vert C\vert + 1 < k$:  \hfill
\textit{($\vert C\vert + 1$ is the current number of clusters detected)}

\begin{enumerate}
    \item [3.1.] Let (\textit{min\_cluster\_size}) $s = f \frac{\vert T'\vert}{\vert C\vert+2}$.

    \item [3.2.] Find the edge $e$ in $T'\setminus C$ of the greatest weight
    such that all connected components in $T'\setminus (C\cup \{e\})$
    are of size at least $s$.

    \item[3.3.] Let $C = C\cup \{e\}$.
\end{enumerate}
\item[4.] Return $T\setminus C$ as a result.
    \hfill $\square$

\end{enumerate}
\end{algorithm}

Lumbermark as given by Algorithm~\ref{algo:lumbermark} ``chops off''
the tree limbs of considerable sizes provided that they are joined by edges that ``stick out''.
Thanks to the nice topological properties of mutual reachability minimum
spanning trees (see, e.g., \citealp{fasthdbscan}), the algorithm is able to extract
clusters of quite different densities and shapes provided that they are well
separated from one another.  Most crucially, \textit{min\_cluster\_size}
depends both on the sample size and the number of clusters in each iteration,
so unlike in HDBSCAN and the like, the algorithm's quality does not degrade
as $n$ increases. Lumbermark is a hierarchical method and thanks to the proposed
way of making MSTs unambiguous, is able to produce any number
of properly nested clusters up to some point allowed by \textit{min\_cluster\_size}.

Leaf stripping makes the algorithm more robust by
not taking points at the cluster borders nor obvious outliers
into account. This step can be made optional, in which case
when the spanning tree is an Euclidean MST and $f=0$, the above algorithm
reduces to the single linkage algorithm.
Otherwise, we can think of it as a robustified version of single linkage,
where points lying in lower-density regions are being pulled away from each other
(due to the reliance on the mutual reachability distance) and where
we prevent the formation of too small clusters
(thanks to the use of \textit{min\_cluster\_size}).
By varying the \textit{min\_cluster\_factor} $f$ factor,
we may allow for smaller or larger inbalacedness in cluster sizes.
From the experiments below, we recommend setting $f=0.25$.

\subsection{Lumbermark vs divisive Genie}

Lumbermark can be thought of as a divisive version
of the Genie algorithm, which is an agglomerative method based on MSTs.
It starts off with $n$ singleton clusters.
In the $j$-th iteration of the algorithm,
there are $k=n-j+1$ clusters; assume they
are of sizes $s_1,\dots,s_{k}$ with $s_1\ge s_2\ge\dots\ge s_{k}$
and note that $\sum_{i=1}^k s_i=n$.
As long as the Gini inequality index
$G(s_1,\dots,s_k)=\left(\sum_{i=1}^k (k-2i+1) s_i\right)/n(k-1)$
is below a given threshold $g\in[0,1]$, two clusters joined by the shortest
MST edge are merged, just like in single linkage.  However, if the cluster size
inequality level increases above $g$, a forced merge of a smallest cluster
with its nearest neighbour is performed. In case of ambiguity, the shortest MST edge available is prioritised.
We should note that this does not guarantee that the Gini index of cluster sizes
is always below $g$: it is merely a merge criterion; compare Figure~3 in the paper \cite{GagolewskiETAL2016:genie}.
It does not even warrant that after a forced merge the Gini index decreases,
e.g., we have $G(15, 14, 13)\simeq 0.0476 < G(17, 15) \simeq 0.0625$.

A straightforward way to turn Genie into a divisive
algorithm would be to split a cluster by selecting
the MST edge of the greatest weight such that the Gini index
of the resulting cluster sizes is kept below $g$ in every iteration.
However, the results generated by such an algorithm are disappointing.
It is because in the case of $k$ clusters, the minimal cluster size
can become unconstrained already when $g\gtrapprox \frac{1}{k-1}$,
which we prove next.

\begin{proposition}
For any $n$ and $k$, let $(s_1,\dots,s_k)$ be a sequence of
cluster sizes such that $s_1\ge\dots\ge s_k$, $\sum_{i=1}^k s_i=n$,
and $G(s_1,\dots,s_k)\ge\frac{1-\frac{k-l(k-l-1)}{n}}{k-1}$.
Then it is possible that $s_k=1$.
\end{proposition}
\begin{proof}
It is known (e.g., \citealp{marshall_olkin}) that the Gini index is a Schur-convex function,
i.e., given two nonincreasing sequences $(s_1,\dots,s_k)$ and $(s_1',\dots,s_k')$ of the same sum $n$,
if for all $l$ we have $\sum_{i=1}^l s_i \le \sum_{i=1}^l s_i'$,
then $G(s_1,\dots,s_k) \le G(s_1',\dots,s_k')$.

Due to Schur-convexity, the nonincreasing integer sequence $(s_1,\dots, s_k)$ of sum $n$
with $s_k=1$ of the smallest Gini index is of the form
$\left(l\ast\left\lfloor \frac{n-1}{k-1}+1\right\rfloor, (k-l-1)\ast\left\lfloor \frac{n-1}{k-1}\right\rfloor, 1\right)$
with $l=(n-1)\,\text{mod}\,(k-1)$, where $l\ast x$ denotes $x$ repeated $l$ times.
Based on the definition of $G$, we have:
\begin{eqnarray*}
(k-1) n G(s_1,\dots,s_k)
& = & \left\lfloor \frac{n-1}{k-1}\right\rfloor \sum_{i=1}^{k-1} (k-2i+1) + \sum_{i=1}^{l} (k-2i+1) - (k-1) \\
& = & (n-1) - (n-1)\,\text{mod}\,(k-1)    + l(k-l) - (k-1) \\
& = & n-k + l(k-l-1).
\end{eqnarray*}
Hence:
\begin{eqnarray*}
G(s_1,\dots,s_k) & = & \frac{1-\frac{k - l(k-l-1)}{n}}{k-1}.
\end{eqnarray*}
\end{proof}

Thus, merely wanting a bound on inequality is not a sensible constraint
as then the algorithm might reduce to single linkage, which often
yields subpar results due to the well-known chaining effect.
Nevertheless, we can show that thanks to \textit{min\_cluster\_size} that
is a function of $n'$ (the number of vertices excluding leaves as per Step 1),
$s=f n'/k$ in Step 3.1 of Lumbermark,
we get the upper bound for the Gini index
of the cluster sizes equal to $g=1-f$.

\begin{proposition}
For any $n'$, $k$, and $f\in[0,1]$, we have $G(s_1,\dots,s_k)\le 1-f$,
where $(s_1,\dots,s_k)$ is a nonincreasing sequence of sum $n'$
with $s_k\ge f \frac{n'}{k}$.
\end{proposition}
\begin{proof}
As mentioned in the previous proof, $G$ is Schur-convex.
Therefore the nonincreasing length-$k$ sum-$n'$ sequence
with minimal element $s=f \frac{n'}{k}$
that has the greatest Gini index
is of the form $(n'-(k-1)s,s,\dots,s)$. Based on the definition of $G$, we have:
\begin{eqnarray*}
(k-1) n' G(s_1,\dots,s_k)
& = & (n'-(k-1)s) (k-1)  + s \sum_{i=2}^k (k-2i+1) \\
& = & (n'-(k-1)s) (k-1)  + s (k-1 + k(k-1)-(k-1)(k+2)) \\
& = & (k-1) (n' - s k).
\end{eqnarray*}
Therefore:
\begin{eqnarray*}
G(s_1,\dots,s_k) &=&
1-s\frac{k}{n'}
= 1-f.
\end{eqnarray*}
\end{proof}

Thus, Lumbermark is a divisive algorithm that puts a constraint
on the level of inequality of cluster sizes.

\subsection{Reference implementation}

The reference C++ implementation of the algorithm is included
in the open-source \textbf{lumbermark}\footnote{See \url{https://lumbermark.gagolewski.com/}.}  v0.9.1 package for Python and R.
The Python version, available via the PyPI repository,
implements a \textbf{scikit-learn}-compatible API \cite{sklearn},
with classes exposing the \textit{fit()} and \textit{fit\_predict()} methods.
Both versions rely on \textbf{quitefastmst} to compute the mutual reachability
minimum spanning trees, which uses an optimised single-tree Borůvka \citeyearpar{Boruvka1926:mst1}
algorithm backed up by K-d trees \cite{kdtrees,kdtreesmidrange,kdtreesoptim}
to speed up parallelised near-neighbour lookups in low-dimensional spaces.

Given an $M$-mutual reachability minimum spanning tree with edges
sorted increasingly which takse between $\Omega(n\log n)$ and $O(n^2)$ time,
the time complexity of Lumbermark is $O(nk)$. This is because in each of the $k$ iterations,
we need to:
\begin{itemize}
\item relabel all the points in the two newly created connected components,
\item determine the total number of vertices on both sides of each edge in the two components,
\end{itemize}
which require only a simple depth-first search.
The memory use is $O(n)$.

We note that if $s$ is too large relative to the size of the dataset,
it might theoretically happen that the algorithm will fail to identify
exactly $k$ clusters. In such a case, the procedure should be restarted
with a smaller \textit{min\_cluster\_factor}. Luckily, this can be performed
without recomputing the MST.

Also, we will assume that Lumbermark with respect to $M=0$,
contrary to $M=1$ (which also relies on raw Euclidean distance),
does not skip MST leaves (Step 1 in Algorithm~\ref{algo:lumbermark}).
It will make it easier to compare it with other algorithms that operate
on unstripped Euclidean MSTs.

\subsection{Example run times}

Table~\ref{tab:runtimes} gives example run times of Lumbermark
with $M=1$ and $M=10$. For the sake of comparison, also the following
algorithms are benchmarked (using default parameters unless stated otherwise):
\begin{itemize}
\item ``standard'' (\textbf{hdbscan} v0.8.44 for Python) and ``fast'' (\textbf{fast\_hdbscan} v0.3.2 for Python; \citealp{fasthdbscan}) HDBSCAN
with the default parameters \textit{min\_samples=min\_cluster\_size=5},
\item Genie (\textbf{genieclust} v1.3.0 for Python; \citealp{genieclust}),
\item K-Means (\textit{init='k-means++'} and \textit{init='random'} with 10 restarts; compare \citealp{Celebi2013200}),
EM for Gaussian mixtures, Ward linkage, and Spectral clustering
(\textbf{scikit-learn} v1.9.1 for Python).
\end{itemize}
Some results are not reported because the computations took significantly
more time than in other cases.

\begin{landscape}
\begin{table*}[p!]
\caption{\label{tab:runtimes} Run times in seconds (6 threads), medians of 5: Detecting $k$ clusters in random datasets with $n$ points in $\mathbb{R}^d$.
The two HDBSCAN implementations do not allow requesting a specific number of clusters $k$ (marked with *).}

\smallskip
\centering
\footnotesize
\begin{tabularx}{0.9\linewidth}{rrrR{4.0em}R{4.0em}R{4.0em}R{4.0em}R{4.0em}R{4.0em}R{4.0em}R{4.0em}R{4.0em}R{4.0em}}
\toprule
    &     &     & \multicolumn{2}{c}{\centering Lumbermark} & Genie & \multicolumn{2}{c}{K-Means} & Gauss- & \multicolumn{2}{c}{HDBSCAN} & & Spec-\\
$d$ & $k$ & $n$ &  M=1 & M=10 & & ++ & random & Mix & fast & std & Ward & tral \\
\midrule
\multirow[t]{6}{*}{2} & \multirow[t]{3}{*}{2} & 10000 & 0.01 & 0.01 & 0.01 & 0.00 & 0.01 & 0.01 & 0.04 & 0.12 & 1.13 & 11.94 \\
 &  & 100000 & 0.06 & 0.07 & 0.09 & 0.02 & 0.06 & 0.13 & 0.63 & 1.50 & --- & --- \\
 &  & 1000000 & 1.07 & 1.06 & 1.62 & 0.16 & 0.61 & 1.06 & 16.25 & --- & --- & --- \\
\midrule
 & \multirow[t]{3}{*}{10} & 10000 & 0.01 & 0.01 & 0.01 & 0.01 & 0.04 & 0.05 & * & * & 1.13 & 11.90 \\
 &  & 100000 & 0.07 & 0.08 & 0.09 & 0.04 & 0.19 & 0.61 & * & * & --- & --- \\
 &  & 1000000 & 1.73 & 1.50 & 1.62 & 0.45 & 2.24 & 5.78 & * & --- & --- & --- \\
\midrule
\multirow[t]{6}{*}{5} & \multirow[t]{3}{*}{2} & 10000 & 0.01 & 0.02 & 0.02 & 0.00 & 0.02 & 0.02 & 0.06 & 0.41 & 1.29 & 12.89 \\
 &  & 100000 & 0.14 & 0.17 & 0.20 & 0.03 & 0.10 & 0.13 & 0.94 & 6.85 & --- & --- \\
 &  & 1000000 & 2.24 & 2.68 & 3.13 & 0.20 & 1.02 & 1.31 & 22.47 & --- & --- & --- \\
\midrule
 & \multirow[t]{3}{*}{10} & 10000 & 0.01 & 0.02 & 0.02 & 0.01 & 0.08 & 0.25 & * & * & 1.30 & 12.93 \\
 &  & 100000 & 0.15 & 0.19 & 0.19 & 0.09 & 0.52 & 5.48 & * & * & --- & --- \\
 &  & 1000000 & 2.87 & 3.10 & 3.11 & 0.88 & 7.04 & 33.43 & * & --- & --- & --- \\
\midrule
\multirow[t]{6}{*}{10} & \multirow[t]{3}{*}{2} & 10000 & 0.09 & 0.09 & 0.09 & 0.01 & 0.03 & 0.02 & 0.25 & 1.03 & 1.45 & 12.54 \\
 &  & 100000 & 1.67 & 1.85 & 1.79 & 0.03 & 0.15 & 0.26 & 4.75 & 62.83 & --- & --- \\
 &  & 1000000 & 43.36 & 49.73 & 44.31 & 0.28 & 1.63 & 1.92 & 202.49 & --- & --- & --- \\
\midrule
 & \multirow[t]{3}{*}{10} & 10000 & 0.09 & 0.09 & 0.09 & 0.02 & 0.15 & 1.13 & * & * & 1.44 & 12.60 \\
 &  & 100000 & 1.70 & 1.86 & 1.79 & 0.19 & 1.45 & 7.87 & * & * & --- & --- \\
 &  & 1000000 & 43.92 & 49.83 & 44.28 & 1.24 & 15.03 & 79.08 & * & --- & --- & --- \\
\bottomrule
\end{tabularx}
\end{table*}
\end{landscape}

The results were obtained on a mid-range laptop running OpenSUSE Tumbleweed GNU/Linux 7.1.4
with Intel Core i5-1335U CPU (10 cores, performance mode) and 32 GiB of RAM. Where applicable,
six threads were utilised. The packages were compiled using the gcc 16.2 compiler flags \textit{-O3 -march=native}.

For each $n$ (number of points), $d$ (dimensionality), and $k$ (cluster count),
five datasets were generated, and the median time (in seconds) is reported.
Data features were sampled independently from a uniform distribution on
the unit interval, with no cluster structure, which is a particularly difficult
case for the distance-based methods because of the curse of dimensionality,
clearly visible for higher values of dataset dimensionality, $d$.

The run times of Lumbermark, Genie, and standard and fast HDBSCAN are mostly affected
by the cost of computing the MST, which is the lowest when the \textbf{quitefastmst} package is utilised.  The HDBSCAN packages
could benefit from relying on this faster implementation.  However, according
to \citet{hdbscan}, the cost of HDBSCAN is still up to $O(n^2)$ because
of the applied relabelling.

Overall, Lumbermark can be considered a fast algorithm
for datasets of low-medium intrinsic dimensionality.
Nevertheless, for higher-dimensional datasets, approximate MSTs could
be considered to speed up the computations. Otherwise, its performance is still
satisfactory taking into account that other algorithms such as Ward linkage
or Spectral clustering can only solve problems of much smaller orders of magnitude
in reasonable time.

\subsection{Quasi-hierarchical version}

Lumbermark produces a hierarchy of nested partitions.
On a side note, we could redefine the algorithm's Step 3.1 as:

\begin{enumerate}
\item [3.1.'] Let (\textit{min\_cluster\_size}) $s = f \frac{\vert T'\vert}{k}$,
\end{enumerate}

\noindent
which leads to a quasi-hierarchical algorithm: the partitions
corresponding to $T\setminus C$ in each iteration of the loop in Step 3
will all be subsets of those from the previous stages, but now
they depend on $k$.  In the \textbf{lumbermark} package,
this version is available by setting \textit{nested=False}.

Such a procedure is slightly faster because the search for the edge of the greatest weight fulfilling the size constraint does not have to be restarted in each iteration.
Furthermore, it does not have the drawback
that if we set $f$ too high and have, e.g., three clusters of almost equal sizes,
then one of them will be split in the first iteration, leading to
suboptimal results. However, from our initial experiments, this version is not significantly better than the original one on average,
therefore we will not benchmark it in this paper to avoid clutter.

We also note that both versions of our algorithm are greedy heuristics to solve
a much harder optimisation problem: find the $k-1$ MST edges
of the greatest sum of weights whose removal yields $k$ connected components
each of the size at least~$s=f n'/k$.

\section{Experiments}\label{sec:experiments}

\subsection{Clustering algorithms for comparison}
In this section, we will test our proposal against a considerable collection of
clustering algorithms that allow setting the desired number of clusters, $k$,
and whose implementations are available in the Python or R environments.
These are:

\begin{itemize}
\item from \textbf{scikit-learn} v1.9.1 for Python (\citealp{sklearn}):
\begin{itemize}
\item Gaussian Mixture (expectation-maximisation for Gaussian mixtures; \textit{n\_init=100, covariance\_type='full'}; e.g.,~\citealp{emalgo}),
\item K-means (best identified parameters: \textit{n\_init=100}; \citealp{lloyd}),
\item Birch  (best identified parameters: \textit{threshold=0.01, branching\_factor=50}; \citealp{ZhangEtAl1996:Birch}),
\item Spectral (K-means over a projection of the normalised Laplacian; best identified parameters: \textit{affinity='laplacian'}, i.e., kernel $K(x, y) = \exp( -\gamma \| x-y \|_1)$, $\gamma=5$, \textit{n\_init=10}; e.g.,~\citealp{spectral}),
\end{itemize}

\item from \textbf{FCPS} v1.3.4 for R (\citealp{ThrunStier2021:fcas}):
\begin{itemize}
\item Adaptive density peaks \citep{RodriguezLaio2014:densitypeak,WangXu2015:adaptivedensitypeak},
\item Clara (Clustering LARge applications;  \citealp{RousseeuwKaufman1990:groupsdata}),
\item Diana (DIvisive ANAlysis clustering;  \citealp{RousseeuwKaufman1990:groupsdata}),
\item Hardcl (On-line update algorithm -- Hard competitive learning; \citealp{Ripley2007:patreg}),
\item Minimax  (\citealp{Murtagh1983:survey,BienTibshirani2011:minimax})
\item MinEnergy (\citealp{SzekelyRizzo2005:minenergy})
\item Softcl  (Neural gas -- Soft competitive learning; \citealp{MartinetzETAL1993:neuralgas}),
\item PAM (Partition around medoids;  \citealp{RousseeuwKaufman1990:groupsdata}),
\end{itemize}

\item from \textbf{fastcluster} v1.1.26 for Python (\citealp{Mullner2011:fastclusteralg}):
\begin{itemize}
\item Average, complete, Ward, centroid, median, McQuitty linkages,
\end{itemize}

\item from \textbf{genieclust} v1.3.0 for Python (\citealp{genieclust}):
\begin{itemize}
\item Genie (Gini index threshold $g=0.3$, smoothing parameters $M=1$ and $M=5$; \citealp{GagolewskiETAL2016:genie}),
\item Single linkage (\citealp{singlelinkage}),
\end{itemize}

\item from \url{https://github.com/lukaszbrzozowski/msts} commit 8536b11:
\begin{itemize}
\item CTCEHC \cite{MA2021194},%
\item HEMST \cite{GrygorashETAL2006:mstcluster},%
\end{itemize}

\item from \url{https://github.com/Layebuniv/SequentialMSTclustering} commit f343c0a:
\begin{itemize}
\item LDPMST \cite{ldpmst},
\item SeqMST (SequentialMST) by A. Layeb,
\end{itemize}

\item from \textbf{dbscan} v1.2.7 for R:
\begin{itemize}
\item HDBSCAN (with different \textit{minPts} settings; this version generates a complete cluster dendrogram that can but cut at an arbitrary level; \citealp{hdbscan_lnai,hdbscan}),
\end{itemize}

\item from \url{https://github.com/liurui39660/SnnDpc} commit 2e75cd3:
\begin{itemize}
\item SNNDPC (best identifier parameter: \textit{k=15}; \citealp{snndpc}),
\end{itemize}

\item from \url{https://github.com/amueller/information-theoretic-mst} commit 178fd43:
\begin{itemize}
\item ITM (\citealp{itm}).
\end{itemize}

\end{itemize}

\subsection{Comparison against 61 datasets with reference partitions}

The setting of the first experiment is identical to the one from \cite{GagolewskiETAL2025:cvimst}.
Namely, we considered a collection of 61 datasets with $n<10{,}000$ and $d\le 3$ stemming from different sources such as
\cite{graves,chameleon,JainLaw2005:dilemma,FrantiVirmajoki2006:ssets,ThrunUltsch2020:fcps}.
We combined them and made available for download from the benchmark dataset repository \cite{clustering_benchmarks} v1.1.0.
All datasets come with reference labels:
in particular, 25 of them have labels by their original authors, which we consider
separately to avoid bias.
Evaluating clustering algorithms is done on low-dimensional data so that the
labels can be inspected visually (there is a huge problem with the validation
of the labels in high dimensional datasets, which we shall point out later).
The underlying assumption is that an algorithm that performs well in 2- or 3-dimensions might be considered in
more complex tasks. On the other hand, if a method does not handle even
such simple scenarios well, it is quite unlikely to be trustworthy in high-dimensional
cases.  To measure the similarity between an algorithm's output and the
reference partition, the adjusted Rand (AR) index \cite{HubertArabie1985:partitionscomp} is used.
In case of datasets with more than one reference labelling, the highest AR
is reported.

\begin{table*}[t!]
\caption{\label{tab:results-compare-all}
Aggregated adjusted Rand indices for different algorithms;
bold items: Wilcoxon signed-rank test results not significant at $\alpha=0.05$
with alternative=greater against Lumbermark\_f0.25\_M5.
}

\centering\footnotesize
\begin{tabular}{lrrrrrrrr}
\toprule
  & $<.8$ & $\ge .95$ & Med & Mean & $<.8$ & $\ge .95$ & Med & Mean \\
\midrule
 & \multicolumn{4}{c}{Third-party labels only (25)} & \multicolumn{4}{c}{All labels (61 instances)} \\
\midrule
Lumbermark\_f0.25\_M5 & 4 & 15 & \bf 0.98 & 0.91 & 11 & 41 & \bf 0.99 & 0.89 \\
Lumbermark\_f0.25\_M0 & 6 & 12 & 0.94 & 0.89 & 14 & 38 & 0.99 & 0.88 \\
Genie\_G0.3\_M5 & 7 & 11 & 0.94 & 0.87 & 18 & 34 & 0.98 & 0.85 \\
Genie\_G0.3\_M0 & 7 & 12 & 0.94 & 0.85 & 18 & 35 & 0.99 & 0.85 \\
Spectral & 8 & 13 & \bf 0.96 & 0.83 & 29 & 25 & 0.88 & 0.69 \\
SNNDPC\_15 & 8 & 7 & 0.93 & 0.80 & 28 & 21 & 0.86 & 0.74 \\
SeqMST & 10 & 11 & 0.93 & 0.79 & 28 & 24 & 0.86 & 0.75 \\
Gaussian Mixture & 9 & 10 & \bf 0.94 & 0.75 & 28 & 22 & 0.86 & 0.67 \\
MinEnergy & 12 & 8 & 0.90 & 0.75 & 38 & 15 & 0.65 & 0.61 \\
ITM & 16 & 7 & 0.77 & 0.73 & 37 & 18 & 0.75 & 0.73 \\
CTCEHC & 10 & 8 & 0.88 & 0.71 & 40 & 11 & 0.63 & 0.61 \\
K-Means & 12 & 9 & 0.82 & 0.71 & 38 & 17 & 0.61 & 0.58 \\
PAM & 13 & 10 & 0.77 & 0.70 & 41 & 15 & 0.61 & 0.57 \\
Birch & 11 & 8 & 0.83 & 0.70 & 36 & 16 & 0.62 & 0.58 \\
Adaptive Density Peaks & 13 & 8 & 0.76 & 0.70 & 41 & 12 & 0.59 & 0.56 \\
Clara & 13 & 6 & 0.75 & 0.70 & 41 & 11 & 0.62 & 0.57 \\
Average Linkage & 11 & 8 & 0.91 & 0.68 & 39 & 15 & 0.53 & 0.54 \\
Ward Linkage & 13 & 6 & 0.72 & 0.68 & 38 & 13 & 0.62 & 0.57 \\
Centroid Linkage & 11 & 8 & 0.87 & 0.65 & 39 & 15 & 0.51 & 0.53 \\
LDPMST & 12 & 8 & 0.80 & 0.64 & 28 & 27 & 0.83 & 0.62 \\
Softcl & 14 & 5 & 0.73 & 0.64 & 39 & 13 & 0.59 & 0.56 \\
Minimax & 13 & 7 & 0.77 & 0.64 & 41 & 11 & 0.45 & 0.51 \\
Hardcl & 16 & 2 & 0.67 & 0.63 & 44 & 6 & 0.58 & 0.53 \\
Complete Linkage & 15 & 6 & 0.78 & 0.63 & 43 & 10 & 0.44 & 0.52 \\
Diana & 14 & 1 & 0.79 & 0.62 & 41 & 7 & 0.53 & 0.53 \\
McQuitty Linkage & 16 & 4 & 0.61 & 0.59 & 45 & 7 & 0.43 & 0.48 \\
Median Linkage & 16 & 5 & 0.71 & 0.56 & 48 & 7 & 0.39 & 0.45 \\
HEMST & 19 & 5 & 0.53 & 0.51 & 50 & 9 & 0.40 & 0.41 \\
Single Linkage & 16 & 8 & 0.44 & 0.48 & 36 & 23 & 0.47 & 0.50 \\
HDBSCAN\_3 & 17 & 7 & 0.23 & 0.43 & 37 & 22 & 0.44 & 0.46 \\
HDBSCAN\_5 & 20 & 4 & 0.03 & 0.26 & 42 & 18 & 0.03 & 0.38 \\
HDBSCAN\_7 & 21 & 3 & 0.03 & 0.24 & 45 & 15 & 0.02 & 0.34 \\
HDBSCAN\_10 & 22 & 3 & 0.00 & 0.18 & 45 & 16 & 0.00 & 0.29 \\
\bottomrule
\end{tabular}

\end{table*}

Table~\ref{tab:results-compare-all} reports the numbers
of cases with AR $<0.8$ and AR $\ge 0.95$ (``bad'' and ``good'' matches)
as well as median and mean AR across the subset of 25 and entirety of 61 datasets.

Compared to Table 4 in \cite{GagolewskiETAL2025:cvimst},
Spectral Clustering now converges on all datasets (results recomputed with
greater numerical precision), which boosted the results of this algorithm.
Moreover, results of HEMST and CTCEHC slightly differ from the previous report
as these methods were sensitive to the preprocessing we applied: adding a tiny
amount of noise to make the MSTs unambiguous.
Just like in the previous study, Softcl did not converge on one dataset
(\textit{sipu/r15}).
We added the results for three new algorithms (SNNDPC, SeqMST, LDPMST).

Overall, the Lumbermark algorithm with the smoothing parameter of $M=5$
and \textit{min\_cluster\_factor} of $f=0.25$ yields
the best results in all the categories, and it is the parameter
setting we recommend for production use.

Moreover, in the case of clustering with respect to the standard Euclidean
distance $M=0$, Lumbermark with $f=0.25$ offers a slight improvement
over the more complicated Genie algorithm, which was the winner of the ranking
presented by \citet{GagolewskiETAL2025:cvimst}.

\subsection{Choice of the smoothing factor $M$}

Let us study the effects of choosing the underlying parameters of the Lumbermark
algorithm in greater detail. For that we will rely only on datasets from
\cite{clustering_benchmarks} which come with unambiguous reference partitions.
Additionally, we excluded the datasets where all the tested clustering procedures
yielded AR $\le$ 0.5.  In other words, we assume that if no algorithm
is able to match the reference labels, then there might be a problem
with the latter (e.g., true clusters are too overlapping).
Including such results in the current setting would only introduce more noise.
To sum up, the 46 datasets that we take into account are:
\textit{FCPS/atom, chainlink, engytime, hepta, lsun, target (labels1), tetra, twodiamonds, wingnut};
\textit{Graves/dense, line, parabolic, ring\_noisy, zigzag\_noisy};
\textit{SIPU/a1, a2, a3, aggregation, compound (labels2), d31, flame (labels1), jain, pathbased, r15, s1, s2, spiral, unbalance};
\textit{WUT/circles, isolation, mk1, mk2, mk3, mk4, smile, stripes, trapped\_lovers, windows, x1, z2, z3};
\textit{Other/chameleon\_t4\_8k, chameleon\_t5\_8k, chameleon\_t8\_8k, hdbscan, square}.

Figure~\ref{fig:heatmap_Lumbermark_AR} gives the proportions of
datasets yielding AR $\ge 0.95$ and $<0.8$ as well as the average AR.
Datasets with balanced and imbalanced reference cluster sizes
(Gini index $>0.2$) are reported separately.

We conclude that small smoothing factors ($M\le 10$) tend to work best.
Overall, the benefit of introducing the mutual reachability distance
compared to the raw Euclidean metric ($M= 1$) is rather modest,
but still significant (compare the Wilcoxon paired test results
reported in Table~\ref{tab:results-compare-all}).
Balanced datasets are handled well by a wide range
of \textit{min\_cluster\_factor}s. Obviously, for imbalanced cluster sizes
we need to decrease this parameter.

We could also be interested in applying the runner-up method on mutual reachability MSTs,
i.e., Genie.  From Figure~\ref{fig:heatmap_Genie_AR}, unless the smoothing factor is too high,
we see that the effect of introducing the $M$ parameter is much smaller than in the previous case.
We also note that Genie is much more sensitive to choice of the underlying
parameter: higher values of the Gini index threshold $G$ do not yield
satisfactory results on balanced datasets, and lower ones are not recommended
when imbalanced clusters are expected.
The Lumbermark method is more robust in that regard.

\begin{figure*}[p!]
\hspace*{-0.35cm}
\includegraphics[width=1.05\linewidth]{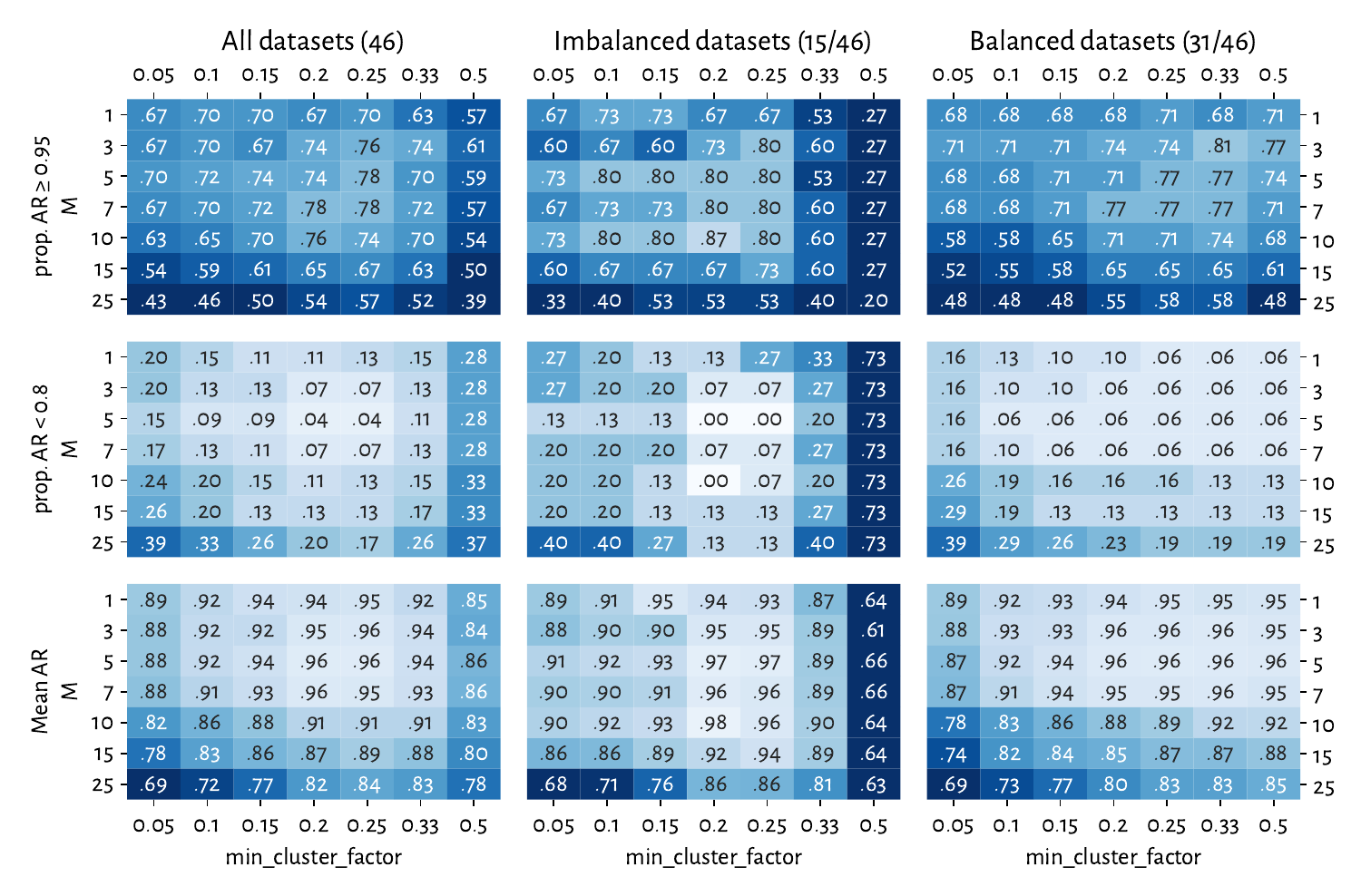}

\caption{\label{fig:heatmap_Lumbermark_AR}
Proportion of datasets with AR $\ge 0.95$ and $<0.8$ and the average AR over 46~datasets
as a function of Lumbermark's \textit{min\_cluster\_factor} $f$ and smoothing parameter~$M$.
}
\end{figure*}

\begin{figure*}[p!]
\hspace*{-0.35cm}
\includegraphics[width=1.05\linewidth]{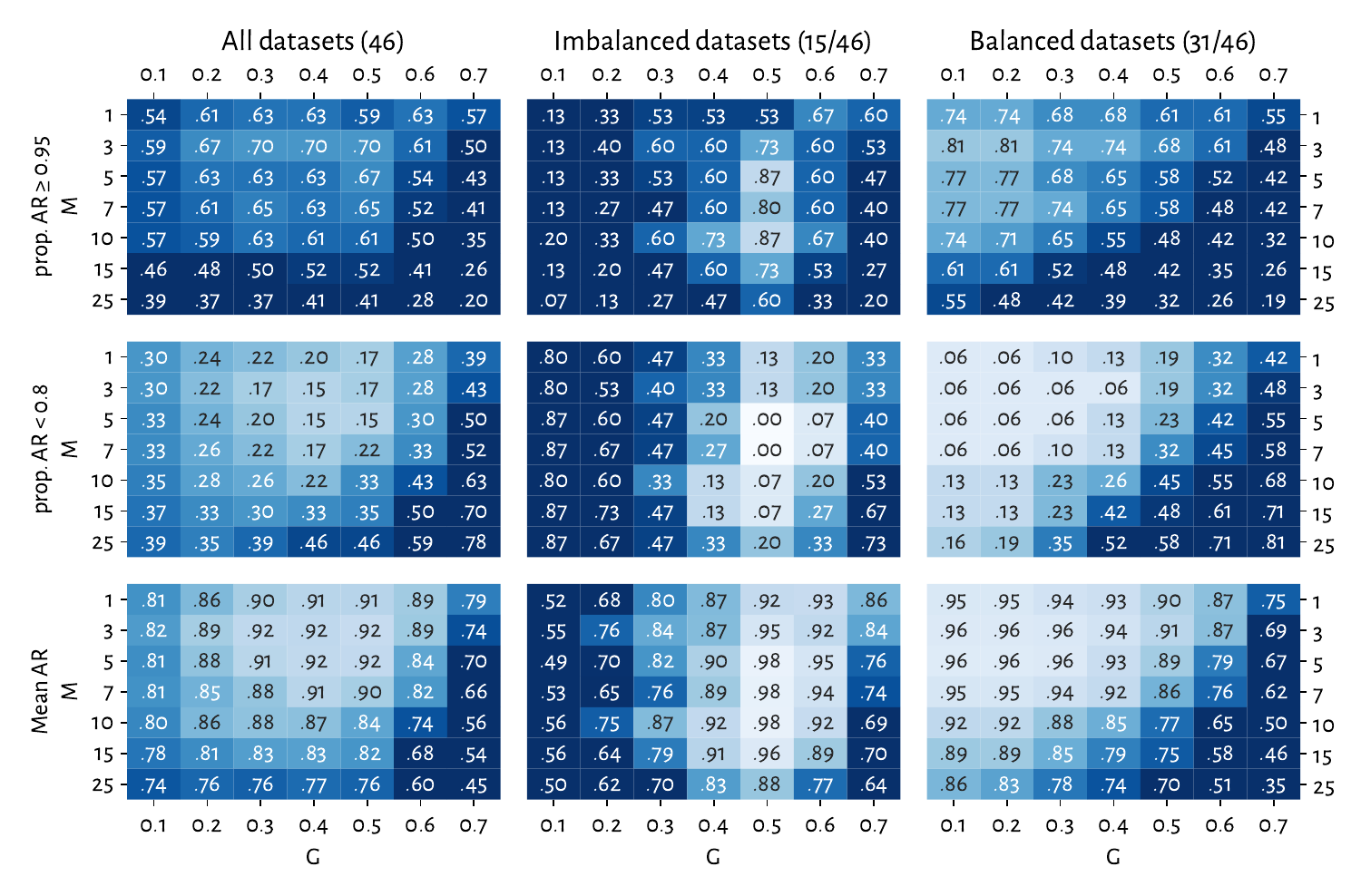}

\caption{\label{fig:heatmap_Genie_AR}
Proportion of datasets with AR $\ge 0.95$ and $<0.8$ and the average AR over 46~datasets
as a function of Genie's Gini index threshold $G$ and smoothing parameter $M$.
}
\end{figure*}

\subsection{The algorithms' performance on individual batteries}

In the previous two sections, we identified the best parameter
settings for Lumbermark on a collection of datasets.
Let us now inspect how the methods perform on each of the
benchmark batteries separately.
For this, we will rely on an extended version of our benchmark repository (v1.2.0),
which now features the CDC battery \cite{cdc}, and small versions
of the MNIST-digits and Fashion datasets (10k test samples).
We will not be restricting ourselves to small low-dimensional datasets anymore.

Overall, we are going to inspect the following benchmark batteries:
\begin{enumerate}
\item CDC \cite{cdc}: DS1 ($999\times 2$), DS10 ($619\times 2$), DS11 ($2930\times 2$), DS12 ($1030\times 2$), DS13 ($998\times 2$), DS15 ($3603\times 2$), DS16 ($356\times 2$), DS17 ($611\times 2$), DS4 ($442\times 2$), DS6 ($1000\times 2$), DS7 ($1487\times 2$), DS8 ($3000\times 2$),

\item FCPS \cite{ThrunUltsch2020:fcps,ThrunStier2021:fcas}: atom ($800\times 3$), chainlink ($1000\times 3$), engytime ($4096\times 2$), hepta ($212\times 3$), lsun ($400\times 2$), target ($770\times 2$), tetra ($400\times 3$), twodiamonds ($800\times 2$), wingnut ($1016\times 2$),

\item Graves \cite{graves}: dense ($200\times 2$), line ($250\times 2$), parabolic ($1000\times 2$), ring\_noisy ($1050\times 2$), zigzag\_noisy ($300\times 2$),

\item Large \cite{ZhangEtAl1996:Birch,SieranojaFranti2019:densitypeaks}: birch1 ($100000\times 2$), birch2 ($100000\times 2$), digits70k ($70000\times 719$), fashion70k ($70000\times 784$), worms\_2 ($105600\times 2$), worms\_64 ($105000\times 64$),

\item MNIST: digits10k ($10000\times 668$), fashion10k ($10000\times 784$),

\item Other \citep{chameleon,hdbscanpkg, bezdek_iris}: chameleon\_t4\_8k ($8000\times 2$), chameleon\_t5\_8k ($8000\times 2$), chameleon\_t7\_10k ($10000\times 2$), chameleon\_t8\_8k ($8000\times 2$), hdbscan ($2309\times 2$), iris ($150\times 4$), square ($1000\times 2$),

\item SIPU (\citealp{kmsix,JainLaw2005:dilemma,FrantiVirmajoki2006:ssets} etc.): a1 ($3000\times 2$), a2 ($5250\times 2$), a3 ($7500\times 2$), aggregation ($788\times 2$), compound ($399\times 2$), d31 ($3100\times 2$), flame ($240\times 2$), jain ($373\times 2$), pathbased ($300\times 2$), r15 ($600\times 2$), s1 ($5000\times 2$), s2 ($5000\times 2$), s3 ($5000\times 2$), s4 ($5000\times 2$), spiral ($312\times 2$), unbalance ($6500\times 2$),

\item UCI \cite{uci}: ecoli ($336\times 7$), glass ($214\times 9$), ionosphere ($351\times 33$), sonar ($208\times 60$), statlog ($2310\times 18$), wdbc ($569\times 30$), wine ($178\times 13$), yeast ($1484\times 8$),

\item WUT: circles ($4000\times 2$), isolation ($9000\times 2$), mk1 ($300\times 2$), mk2 ($1000\times 2$), mk3 ($600\times 3$), mk4 ($1500\times 3$), olympic ($5000\times 2$), smile ($1000\times 2$), stripes ($5000\times 2$), trajectories ($10000\times 2$), trapped\_lovers ($5000\times 3$), twosplashes ($400\times 2$), windows ($2977\times 2$), x1 ($120\times 2$), z1 ($192\times 2$), z2 ($900\times 2$), z3 ($1000\times 2$).
\end{enumerate}
There are 82 datasets in total.
The omitted datasets are:
\textit{WUT/cross,graph,x2,x3,labirynth} and \textit{Graves/fuzzyx} (overlapping clusters,
it is hard to say what the true cluster structure is),
\textit{Graves/ring,ring\_outliers,zigzag,zigzag\_outliers}
(too similar to \textit{Graves/ring\_noisy} and \textit{zigzag\_noisy}),
\textit{other/iris5} (very imbalanced clusters: sizes 50, 50, 5).
Moreover, in \textit{FCPS/target} we treat the 12 outliers as noise and not real clusters,
in \textit{SIPU/compound} we mark the outliers surrounding the right cluster as noise,
and in \textit{SIPU/flame} we treat the ambiguous points between the two clusters as noise.
In six other cases, we drop the third-party labels and use the original ones so that
for each dataset, there is only one reference way to cluster it.
When reference labels feature noise marks, such points are not included only in the computation of ARI,
but still take part in the clustering process.

Table~\ref{Tab:results_by_bettery_mean}
displays the relative average AR indices for each clustering method and each of the 9 benchmark batteries.
The value of 1 corresponds to the case where the average AR is the highest in a given category.

We note that the Lumbermark algorithm performs very well on all batteries except MNIST (which would require setting $f=0.5$)
and UCI. However, overall the best average ARI for these batteries is still quite low,
which might indicate a problem with the reference labels:
For instance, the best performing method on UCI, Gaussian Mixtures, has average AR of 0.41 whilst Lumbermark
($f=0.25, M=5$) exhibits average AR of 0.24, but both indicate very poor clustering performance (with respect to the currently used reference labels).

Some of the methods were too slow to run on larger ($>10{,}000$) datasets or used too much memory;
they are marked with dashes in the ``Large'' column.  Note that the ``Large'' battery
features the extended versions of the two datasets from the MNIST set.
These images should probably not be clustered merely using the Euclidean distance
on the raw input space.

\begin{table}[t!]
\caption{\label{Tab:results_by_bettery_mean} Average AR indices relative to the best average AR for each benchmark battery.
Note that the maximal average AR for MNIST and UCI are very low, indicating that reference labels therein might not be meaningful.
}
\centering
\footnotesize

\hspace*{-0.8cm}
\begin{tabular}{lrrrrrrrrr}
\toprule
battery & CDC & FCPS & Graves & Large & MNIST & Other & SIPU & UCI & WUT \\
Count   & 12 & 9 & 5 & 6 & 2 & 7 & 16 & 8 & 17 \\
Max Ave. ARI & 0.98 & 0.98 & 0.99 & 0.61 & 0.55 & 0.94 & 0.90 & 0.41 & 0.86 \\
\midrule
method &  &  &  &  &  &  &  &  &  \\
\midrule
Adaptive Density Peaks & 0.82 & 0.58 & 0.35 & --- & 0.00 & 0.50 & 0.84 & 0.54 & 0.56 \\
Average Linkage        & 0.74 & 0.54 & 0.32 & --- & 0.17 & 0.58 & 0.85 & 0.34 & 0.50 \\
Birch                  & 0.72 & 0.65 & 0.35 & 0.45 & 0.86 & 0.58 & 0.80 & 0.64 & 0.54 \\
CTCEHC                 & 0.79 & 0.75 & 0.54 & 0.23 & 0.58 & 0.65 & 0.77 & 0.68 & 0.52 \\
Centroid Linkage       & 0.73 & 0.54 & 0.34 & 0.57 & 0.00 & 0.60 & 0.80 & 0.25 & 0.50 \\
Clara                  & 0.67 & 0.62 & 0.36 & --- & 0.46 & 0.58 & 0.80 & 0.69 & 0.51 \\
Complete Linkage       & 0.69 & 0.54 & 0.20 & --- & 0.38 & 0.49 & 0.74 & 0.44 & 0.52 \\
Diana                  & 0.67 & 0.62 & 0.34 & --- & 0.32 & 0.49 & 0.69 & 0.51 & 0.54 \\
Gaussian Mixture       & 0.81 & 0.76 & 0.59 & --- & 0.73 & 0.63 & 0.81 & 1\phantom{.00} & 0.61 \\
Genie\_G0.3\_M1        & 0.86 &  0.99 & 0.97 & 0.85 & 0.39 & 0.90 & 0.93 & 0.69 & 0.92 \\
Genie\_G0.3\_M5        & 0.86 & { 1}\phantom{.00} & 1\phantom{.00} & 0.83 & 0.15 & 0.89 & 0.94 & 0.86 & 0.89 \\
HEMST                  & 0.69 & 0.52 & 0.03 & 0.28 & 0.00 & 0.40 & 0.51 & 0.22 & 0.40 \\
Hardcl                 & 0.68 & 0.60 & 0.35 & --- & 0.65 & 0.51 & 0.70 & 0.64 & 0.46 \\
ITM                    & 0.62 & 0.98 & 0.64 & 0.99 & 1\phantom{.00} & 0.79 & 0.73 & 0.69 & 0.79 \\
K-Means                & 0.71 & 0.63 & 0.34 & 1\phantom{.00} & 0.66 & 0.57 & 0.82 & 0.69 & 0.54 \\
LDPMST                 & 0.66 & 0.75 & 1\phantom{.00} & --- & 0.01 & 0.19 & 0.69 & 0.43 & 0.76 \\
Lumbermark\_f0.25\_M0  & 1\phantom{.00} &  0.99 & 0.97 & 0.94 & 0.34 & 0.98 & 0.97 & 0.63 & 1\phantom{.00} \\
Lumbermark\_f0.25\_M5  & { 1}\phantom{.00} & { 1}\phantom{.00} & 1\phantom{.00} & 0.96 & 0.42 & 1\phantom{.00} & 1\phantom{.00} & 0.59 & 0.98 \\
McQuitty Linkage       & 0.65 & 0.51 & 0.22 & --- & 0.30 & 0.43 & 0.70 & 0.38 & 0.51 \\
Median Linkage         & 0.59 & 0.42 & 0.24 & 0.46 & 0.00 & 0.43 & 0.71 & 0.55 & 0.47 \\
MinEnergy              & 0.74 & 0.76 & 0.33 & --- & 0.90 & 0.62 & 0.81 & 0.69 & 0.56 \\
Minimax                & 0.72 & 0.52 & 0.26 & --- & 0.44 & 0.53 & 0.80 & 0.41 & 0.52 \\
PAM                    & 0.65 & 0.62 & 0.35 & --- & 0.54 & 0.57 & 0.82 & 0.67 & 0.50 \\
HDBSCAN\_3       & 0.52 & 0.51 & 0.20 & --- & 0.00 & 0.24 & 0.37 & 0.03 & 0.80 \\
HDBSCAN\_5       & 0.44 & 0.40 & 0.20 & --- & 0.00 & 0.24 & 0.21 & 0.02 & 0.73 \\
HDBSCAN\_7       & 0.38 & 0.40 & 0.20 & --- & 0.00 & 0.24 & 0.18 & 0.02 & 0.59 \\
SeqMST                 & 0.59 & 0.94 & 0.81 & --- & 0.47 & 0.67 & 0.84 & 0.53 & 0.67 \\
SNNDPC\_15             & 0.72 & 0.86 & 0.72 & --- & 0.54 & 0.74 & 0.89 & 0.51 & 0.69 \\
Single Linkage         & 0.55 & 0.56 & 0.20 & 0.12 & 0.00 & 0.24 & 0.43 & 0.02 & 0.83 \\
Softcl                 & 0.68 & 0.63 & 0.32 & --- & 0.66 & 0.58 & 0.71 & 0.66 & 0.55 \\
Spectral               & 0.80 & 0.84 & 0.75 & --- & 0.00 & 0.56 & 0.87 & 0.28 & 0.62 \\
Ward Linkage           & 0.77 & 0.60 & 0.35 & 0.95 & 0.86 & 0.64 & 0.80 & 0.61 & 0.53 \\
\bottomrule
\end{tabular}
\end{table}

\section{Conclusion}\label{sec:conclusion}

We have introduced a new divisive clustering algorithm based
on mutual reachability spanning trees named Lumbermark.
Its appealing simplicity comes with fast run times, especially in datasets of
low intrinsic dimensionality, and high-quality outputs.

Lumbermark is robust in the presence of outliers and clusters of imbalanced
sizes or different densities.  It provides an alternative
to HDBSCAN in the case where the number of clusters is known in advance.
Still, as a hierarchical algorithm, it enables us to explore nested partitions
of different cardinalities too.  We have argued why Lumbermark can
be viewed as a divisive version of Genie.  One crucial property is
that the cluster sizes are bound from below by the number of input data points
and are a function of the level of the hierarchy, and not fixed a priori,
which makes it more scalable.

In the course of the algorithm's development, we had to address
the fact that mutual reachability spanning trees are often ambiguous.
Interestingly, we have discovered that the effect of the underlying smoothing
parameter $M$ is modest at most.  Setting $M$ too large is often
counterproductive.  It is likely that this regularity will be observed
when mutual reachability distance is plugged into other MST-based algorithms
\cite{GagolewskiETAL2025:cvimst}, including HDBSCAN, a study which we shall explore more deeply in future research.
Further work on the algorithm will also involve the evaluation of the use of
approximate minimum spanning trees (e.g., \citealp{azarmehr2025massively})
and different disambiguation
methods for the tree building based on mutual reachability distances.

\bigskip
\paragraph{Data and code availability.}
Benchmark data \cite{clustering_benchmarks} can be fetched from {\url{https://clustering-benchmarks.gagolewski.com/}} and
the computed partitions can be downloaded from {\url{https://github.com/gagolews/clustering-results-v1}}.

The \textbf{lumbermark} package for Python and R is distributed under the
GNU AGPL v3 license and is available via the PyPI and CRAN repositories.  The project homepage
is located at \url{https://lumbermark.gagolewski.com/}.

\paragraph{No conflict of interest.}
The author certifies that he has no affiliations with or involvement in any
organisation or entity with any financial interest or non-financial interest
in the subject matter or materials discussed in this manuscript.

\paragraph{No use of generative AI.}
The author did not rely on any generative ``AI'' tool in the writing
of accompanying code nor in the preparation of this manuscript.


\end{document}